\documentclass{article}
\usepackage[letterpaper, margin=0.8in]{geometry}
\usepackage[parfill]{parskip} 
\usepackage[T1]{fontenc}
\usepackage{amsmath, amsthm, mathtools, amssymb}
\usepackage{graphicx}
\usepackage{tikz}
\usepackage[colorlinks=true,citecolor=blue]{hyperref}
\usepackage[capitalize]{cleveref}
\usepackage{bbm}

\usepackage[style=alphabetic,maxnames=99,minnames=3,maxalphanames=3,minalphanames=3]{biblatex}
\allowdisplaybreaks[1] 

\newtheorem{theorem}{Theorem}[section]
\newtheorem{lemma}[theorem]{Lemma}

\newtheorem{corollary}[theorem]{Corollary}

\newtheorem{proposition}[theorem]{Proposition}

\newtheorem{definition}[theorem]{Definition}
\newtheorem{example}[theorem]{Example}

\newcommand{\calD}{{\mathcal{D}}}
\newcommand{\calF}{{\mathcal{F}}}
\newcommand{\calG}{{\mathcal{G}}}
\newcommand{\calH}{{\mathcal{H}}}

\newcommand{\RR}{\mathbb{R}}
\newcommand{\BB}{\mathbb{B}}
\renewcommand{\SS}{\mathbb{S}}

\newcommand{\LS}{\mathsf{LS}} 
\newcommand{\PI}{\mathsf{PI}} 
\newcommand{\RT}{\mathsf{RT}} 

\DeclareMathOperator*{\argmax}{arg\,max}

\DeclareMathOperator*{\E}{\mathbb{E}}
\DeclareMathOperator*{\Var}{\mathrm{Var}}
\DeclareMathOperator*{\Cov}{\mathrm{Cov}}
\DeclareMathOperator*{\Indicator}{\mathbbm{1}}
\DeclareMathOperator*{\bias}{\mathrm{bias}}
\DeclareMathOperator*{\variance}{\mathrm{variance}}

\newcommand{\leqsigma}{{{\scriptscriptstyle\leq} \sigma}}
\newcommand{\leqtau}{{{\scriptscriptstyle\leq} \tau}}
\newcommand{\fhat}{\hat{f}}
\newcommand{\indicatorf}{{\textstyle \Indicator_{\hspace{-1pt}f}}}

\title{Metric-Fair Classifier Derandomization}

\author{Anonymous Authors}

\author{
Jimmy Wu%
\thanks{Email: \texttt{jimmywu126@gmail.com}}
\and 
Yatong Chen%
\thanks{University of California, Santa Cruz, Department of Computer Science and Engineering. Email: \texttt{ychen592@ucsc.edu}}
\and
Yang Liu%
\thanks{University of California, Santa Cruz, Department of Computer Science and Engineering. Email: \texttt{yangliu@ucsc.edu}}
}
\date{\today}

\begin{document}
\maketitle

\begin{abstract}
We study the problem of \emph{classifier derandomization} in machine learning: given a stochastic binary classifier $f: X \to [0,1]$, sample a deterministic classifier $\fhat: X \to \{0,1\}$ that approximates the output of $f$ in aggregate over any data distribution. Recent work revealed how to efficiently derandomize a stochastic classifier with strong output approximation guarantees, but at the cost of individual fairness --- that is, if $f$ treated similar inputs similarly, $\fhat$ did not. In this paper, we initiate a systematic study of classifier derandomization with metric fairness guarantees. We show that the prior derandomization approach is almost maximally metric-unfair, and that a simple ``random threshold'' derandomization achieves optimal fairness preservation but with weaker output approximation. We then devise a derandomization procedure that provides an appealing tradeoff between these two: if $f$ is $\alpha$-metric fair according to a metric $d$ with a locality-sensitive hash (LSH) family, then our derandomized $\fhat$ is, with high probability, $O(\alpha)$-metric fair and a close approximation of $f$. We also prove generic results applicable to all (fair and unfair) classifier derandomization procedures, including a bias-variance decomposition and reductions between various notions of metric fairness.
\end{abstract}
\section{Introduction}
We study the general problem of \emph{derandomizing} stochastic classification models. Consider a typical binary classification setting defined by a feature space $X \subseteq \mathbb{R}^n$ and labels $\{0,1\}$; we wish to devise a procedure that, given a \emph{stochastic} or \emph{randomized} classifier $f: X \to [0,1]$, efficiently samples a \emph{deterministic} classifier $\fhat: X \to \{0,1\}$ from some family of functions $\calF$, such that $\fhat$ preserves various qualities of $f$.

Stochastic classifiers arise naturally in both theory and practice. For example, they are frequently the solutions to constrained optimization problems encoding complex evaluation metrics \cite{narasimhan2018learning}, group fairness \cite{grgic2017fairness,agarwal2018reductions}, individual fairness \cite{dwork2012fairness,rothblum2018probably,kim2018fairness,sharifi2019average}, and robustness to adversarial attacks \cite{pinot2019theoretical,cohen2019certified,pinot2020randomization,braverman2020role}. Stochastic classifiers are also the natural result of taking an ensemble of individual classifiers \cite{dietterich2000ensemble,grgic2017fairness}.

However, they may be undesirable for numerous reasons: a stochastic classifier is not robust to repeated attacks, since even one that is instance-wise 99\% accurate will likely err after a few hundred attempts; by the same token, they violate intuitive notions of fairness since even the \emph{same} individual may be treated differently over multiple classifications. For these reasons, Cotter, Gupta, and Narasimhan \cite{cotter2019making} recently presented a procedure for derandomizing a stochastic classifier while approximately preserving the outputs of $f$ with high probability. However, the authors observe that their construction results in similar individuals typically being given very different predictions --- in other words, it does not satisfy \emph{individual fairness} --- and ask whether it is possible to obtain a family of deterministic classifiers that preserves both aggregate outputs and individual fairness.

Another motivation for studying individually fair decision making comes from the game-theoretic setting of \emph{strategic classification}, wherein decision subjects may modify their features to obtain a desired outcome from the classifier \cite{hardt2016strategic,cai2015optimum,chen2018strategyproof,dong2018strategic,chen2020learning}. A metric-fair stochastic classifier --- and by extension, a metric-fair derandomization procedure --- offers significant protection against such manipulations. See \cref{appendix-section:manipulation-deterrence-in-strategic-classification} for more on this topic.

\subsection{Our Contributions}
In this paper, we initiate a systematic study of classifier derandomization with individual fairness preservation. In line with many recent works, we formalize individual fairness as \emph{metric fairness}, which requires the classifier to output similar predictions on close point pairs in some metric space $(X,d)$ \cite{dwork2012fairness,kim2018fairness,friedler2016impossibility}. Roughly, $f$ is metric-fair if there are constants $\alpha, \beta > 0$ such that for all $x,x' \in X$,
\begin{align*}
    \left|f(x)-f(x')\right| \leq \alpha \cdot d(x,x') + \beta
\end{align*}
A sampled deterministic classifier $\fhat \sim \calF$ is metric-fair when this inequality holds in expectation.

Under this formalism, we obtain the following results:
\begin{enumerate}
    \item We make precise the observation of \cite{cotter2019making} that their derandomization procedure, based on pairwise-independent hash functions, does not preserve individual fairness. In fact, we prove that it is almost \emph{maximally} metric-unfair regardless of how fair the original stochastic classifier was (\Cref{subsection:pairwise-independent-derandomization}).
    
    \item We demonstrate that a very simple derandomization procedure, based on setting a single random threshold $r \sim [0,1]$, attains near-perfect expected fairness preservation, and prove that no better fairness preservation is possible (\Cref{subsection:random-threshold-derandomization}). However, this procedure's output approximation has higher variance than the pairwise-independent hashing approach in general.
    
    \item We devise a derandomization procedure that achieves nearly the best of both worlds, preserving aggregate outputs with high probability, with only modest loss of metric fairness (\Cref{section:fair-derandomization-via-locality-sensitive-hashing}). In particular, when $f$ has fairness parameters $(\alpha,\beta)$, sampling $\fhat$ from our family $\calF_\LS$ yields expected fairness parameters at most $(\alpha + \frac{1}{2},\beta + \epsilon)$. We also show a high-probability aggregate fairness guarantee: \emph{most} deterministic classifiers in $\calF$ assign \emph{most} close pairs the same prediction. These guarantees hold for the class of metrics $d$ that possess locality-sensitive hashing (LSH) schemes, which includes a wide variety of generic and data-dependent metrics.
    
    \item We prove structural lemmas applicable to all classifier derandomization procedures: first, a bias-variance decomposition for the error of a derandomization $\fhat$ of $f$; second, a set of reductions showing that metric fairness-preserving derandomizations also preserve notions of \emph{aggregate} and \emph{threshold} fairness.
\end{enumerate}

A practically appealing aspect of our LSH-based derandomization method is that it is completely oblivious to the original stochastic classifier, in that it requires no knowledge of how $f$ was trained, and its fairness guarantee holds for whatever fairness parameters $f$ happens to satisfy on each pair $(x,x') \in X^2$. The technique can therefore be applied as an independent post-processing step --- for example, on the many fair stochastic classifiers detailed in recent works \cite{rothblum2018probably,kim2018fairness}. The burden on the model designer is thus reduced to selecting an LSHable metric feature space $(X,d)$ that is appropriate for the classification task.

\subsection{Preliminaries}
\label{subsection:preliminaries}
Given a stochastic classifier $f: X \to [0,1]$ and distance function $d: X \times X \to [0,1]$, we wish to design an efficiently sampleable set $\calF$ of deterministic binary classifiers $\fhat: X \to \{0,1\}$; we call $\calF$ a \emph{family of deterministic classifiers}, or a \emph{derandomization of $f$}. Moreover, we would like $\calF$ to have the following properties:

\paragraph{Output approximation:}
$\fhat$ sampled uniformly\footnote{In this paper, we will always sample uniformly from families of classifiers and hash functions; thus $\fhat \sim \calF$ means $\fhat \sim \text{Unif}(\calF)$, and $h \sim \calH$ means $h \sim \text{Unif}(\calH)$.} from $\calF$ simulates or approximates $f$ in aggregate over any distribution. More precisely, define the \emph{pointwise} bias and variance of $\fhat$ with respect to $f$ on a sample $x \in X$ as
\begin{align*}
    \bias(\fhat,f,x)
    := \E_{\fhat \sim \calF} \left[\fhat(x)\right]
        - f(x)
    \qquad \mathrm{and} \qquad
    \variance(\fhat,x)
    := \Var_{\fhat \sim \calF} \left(
            \fhat(x)
        \right)
\end{align*}
Now let $\calD$ be a distribution over $X$. The \emph{aggregate} bias and variance of $\fhat$ with respect to $f$ on $\calD$ are
\begin{align*}
    \bias(\fhat,f,\calD)
    := \E_{x \sim \calD} \left[
            \bias(\fhat,f,x)
        \right]
    \qquad \mathrm{and} \qquad
    \variance(\fhat,\calD)
    := \Var_{\fhat \sim \calF} \left(
            \E_{x \sim \calD} \left[\fhat(x)\right]
        \right)
\end{align*}

We seek a family $\calF$ for which both of these quantities are small. This is a useful notion of a good approximation of $f$ since in practice, classifiers are typically applied \emph{in aggregate} on some dataset or in deployment. In \cref{subsection:output-approximation-and-loss-approximation} we also point out that low bias and variance in the above sense implies that $\fhat$ and $f$ are nearly indistinguishable when compared according to any binary loss functions, such as accuracy, false positive rate, etc.

\paragraph{Individual fairness:}
Similar individuals are likely to be treated similarly. We formally define this notion as \emph{metric fairness}, which says that that the classifier should be an approximately Lipschitz-continuous function relative to a given distance metric:

\begin{definition}[$(\alpha,\beta,d)$-metric fairness] \label{definition:metric-fairness}
Let $\alpha \geq 1$\footnote{We enforce $\alpha \geq 1$, and not merely $\alpha \geq 0$, so that the codomain of $f$ is $[0,1]$ rather than potentially $[0,\alpha]$ (or some other interval of length $\alpha < 1$). Requiring $\alpha \geq 1$ thus makes $f$ a proper stochastic classifier and enables direct comparisons between different fairness parameters. This is no loss of generality since $(\alpha,\beta,d)$-fairness for $\alpha < 1$ can also be expressed as $(1,\frac{\beta}{\alpha},\frac{d}{\alpha})$-fairness or, with some loss of generality, $(1,\beta+\alpha,d)$-fairness.} and $\beta \geq 0$, let $d: X^2 \to [0,1]$ be a metric, and let $x,x' \in X$. We say a stochastic classifier $f: X \to [0,1]$ satisfies \emph{$(\alpha,\beta,d)$-metric fairness on $(x,x')$}, or is \emph{$(\alpha,\beta,d)$-fair on $(x,x')$}, if
\begin{align}
    \left|f(x)-f(x')\right| \leq \alpha \cdot d(x,x') + \beta
        \label{equation:stochastic-individual-fairness-condition}
\end{align}
Similarly, a deterministic classifier family $\calF$ is $(\alpha,\beta,d)$-fair on $(x,x')$ if
\begin{align}
    \E_{\fhat \sim \calF} \left[ \left|\fhat(x) - \fhat(x')\right| \right] \leq \alpha \cdot d(x,x') + \beta
        \label{equation:deterministic-individual-fairness-condition}
\end{align}
When this condition is satisfied for all $(x,x') \in X^2$, we simply say the classifier (or family) is \emph{$(\alpha,\beta,d)$-fair}.
\end{definition}

To intuit this definition, notice that when a classifier satisfies metric fairness with $\beta = 0$, the difference between its predictions on some pair of points $x$ and $x'$ scales in proportion to their distance. To conform to this idea of fairness, it is important that the derandomization procedures we design do not substantially increase these fairness parameters, but especially $\beta$.

The above definition of metric fairness is most closely related to those of Rothblum and Yona \cite{rothblum2018probably}, whose focus is learning a ``probably approximately metric-fair'' model that generalizes to unseen data; and Kim, Reingold, and Rothblum \cite{kim2018fairness}, whose focus is in-sample learning when the metric $d$ is not fully specified. Both works take inspiration from the metric-based notion of individual fairness introduced in \cite{dwork2012fairness}. Crucially however, the aforementioned works provide guarantees exclusively for stochastic classifiers, and to our knowledge, this is the case for all papers to date whose focus is learning metric-fair classifiers.

In addition to this pairwise notion of metric fairness, we will also develop \emph{aggregate} fairness guarantees for various derandomization procedures. To that end, let $X^2_\leqtau := \left\{(x,x') \in X^2 ~\middle|~ d(x,x') \leq \tau\right\}$ denote the set of point pairs within some distance $\tau \in [0,1]$. Our aggregate fairness bounds will state that, with high probability over the sampling of $\fhat \sim \calF$, most pairs $(x,x') \in X^2_\leqtau$ receive the same prediction from $\fhat$.
\section{Output Approximation Versus Fairness}
\label{section:output-approximation-versus-fairness}
We begin our study of metric-fair classifier derandomization by contrasting two approaches: first, the ``pairwise-independent'' derandomization of \cite{cotter2019making}, which achieves a low-variance approximation of the original stochastic classifier, but does not preserve metric fairness; and second, a simple ``random threshold'' derandomization that perfectly preserves metric fairness, at the cost of higher output variance.

\subsection{Pairwise-Independent Derandomization}
\label{subsection:pairwise-independent-derandomization}
The construction of Cotter, Narasimhan, and Gupta \cite{cotter2019making} makes use of a pairwise-independent hash function family $\calH_\PI$, i.e. a set of functions $h_\PI: B \to [k]$ such that
\begin{align*}
    \Pr_{h \sim \calH_\PI} [h(b)=i, h(b')=j] = \frac{1}{k^2}
    \quad
    \forall b \neq b' \in B, \ \ i,j \in [k]
\end{align*}
Observe that a family that satisfies this property is also uniform, i.e. $\Pr_{h \sim \calH_\PI} [h(b)=i] = 1/k$ for all $b,i$.

The classifier family they propose is then\footnote{For the sake of clearer exposition, we simplify the deterministic classifier used in \cite{cotter2019making}, which is actually $\fhat_{h_\PI}(x) := \Indicator\{f(x) \geq \frac{2h_\PI(x)-1}{2k}\}$; this does not change \Cref{theorem:bias-and-variance-of-pairwise-independent-derandomization} or \Cref{proposition:unfairness-of-pairwise-independent-derandomization} beyond a $1/2k$ additive difference in the bias, variance, and $\beta$.}
\begin{align} \label{equation:PI-classifier-family}
    \calF_\PI
    := \left\{\fhat_{h_\PI} ~\middle|~ h_\PI \in \calH_\PI\right\}, 
    \quad \text{where} \quad
    \fhat_{h_\PI}(x)
    := \Indicator\left\{
            f(x) \geq \frac{h_\PI(\pi(x))}{k}
        \right\}
\end{align}
where $\pi: X \to B$ is some fixed \emph{bucketing} function that discretizes the input (since the pairwise-independent hash family has finite domain).

Let us develop some intuition for this construction. First, thinking of $k$ as large, each $\fhat_{h_\PI} \in \calF_\PI$ essentially assigns a pseudo-random threshold $\frac{h_\PI(\pi(x))}{k} \in [0,1]$ to each input $x$, so that $\fhat{(x)} = 1$ if and only if $f(x)$ exceeds the threshold. Since $h_\PI$ is a uniform hash function family, $h_\PI(\pi(x))$ is uniform over $[k]$; this endows $\calF_\PI$ with low bias with respect to $f$. Using this idea and the pairwise-independence of $\calH_\PI$, the authors show that this classifier family exhibits low bias and variance of approximation:

\begin{theorem}[Bias and variance of pairwise-independent derandomization \cite{cotter2019making} (simplified)]
\label{theorem:bias-and-variance-of-pairwise-independent-derandomization}
Let $f$ be a stochastic classifier, $\calD$ a distribution over $X$, and $\pi: X \to B$ a bucketing function. Then $\fhat \sim \calF_\PI$ satisfies
\begin{align*}
    \bias(\fhat_\PI,f,\calD)
    \leq \frac{1}{k}
    \qquad \mathrm{and} \qquad
    \variance(\fhat_\PI,f,\calD)
    \leq \max_{b \in B} \Pr_{x \sim \calD} [\pi(x) = b]
        \cdot \E_{x \sim \calD} [f(x)(1-f(x))]
        + \frac{1}{k}
\end{align*}
Moreover, $\hat{f}_\PI$ can be sampled using $O(\log |B| + \log k)$ uniform random bits.
\end{theorem}

To understand this variance bound, observe that for a given data distribution $\calD$, the bound is stronger or weaker depending on how well $\pi$ disperses samples into different buckets in $B$. When there exists some $b \in B$ such that $\Pr_{x \sim \calD}[\pi(x)=b] \approx 1$, $\variance(\fhat_\PI,f,\calD) \approx \E_{x \sim \calD} [f(x)(1-f(x))]$ essentially tracks the stochasticity of $f$. At the other extreme when $\Pr_{x \sim \calD}[\pi(x)=b] = 1/|B|$ for all $b \in B$, $\variance(\fhat_\PI,f,\calD) \approx 1/|B|$.

As the authors pointed out (but did not formalize), $\fhat_\PI$ does not preserve pairwise fairness in general. We make this observation precise by showing that it is always possible to design a dataset, of any desired size, such that the pairwise-independent derandomization treats \emph{every} pair of points unfairly for nearly any $\beta < 1/2$.

\begin{proposition}[Unfairness of pairwise-independent derandomization]
\label{proposition:unfairness-of-pairwise-independent-derandomization}
For every $N \geq 2$, $\alpha \geq 1$, $\beta < \frac{1}{2}-\frac{1}{2k}$, and metric $d: \RR^n \times \RR^n \to [0,1]$, there exist a set $X \subset \RR^n$ of size $N$ and stochastic classifier $f: X \to [0,1]$ such that the following hold:
\begin{enumerate}
    \item $f$ is nontrivial and $(1,0,d)$-fair.
    \item $\calF_\PI$ violates $(\alpha,\beta,d)$-metric fairness for \emph{every} pair $(x,x') \in X^2, x \neq x'$.
\end{enumerate}
\end{proposition}

If $k$ is not too small, this says that derandomizing using pairwise-independent hashing is almost maximally unfair, as a uniform random binary function $\hat{g}: X \to \{0,1\}$ satisfies $\E[|\hat{g}(x) - \hat{g}(x')|] = 1/2$, and therefore achieves $\beta = 1/2$.

\begin{proof}[Proof sketch of \Cref{proposition:unfairness-of-pairwise-independent-derandomization}]
Consider any $\alpha \geq 1$, $\beta \in \left(0,\frac{1}{2}-\frac{1}{2k}\right)$, and $N \geq 2$. We choose $X$ to be some set of $N$ points on a sufficiently small sphere about the origin, and let $f$ be a classifier that maps half of the points in $X$ to $\frac{1 + \epsilon}{2}$ and the other half to $\frac{1 - \epsilon}{2}$. When $\epsilon > 0$ is sufficiently small, it can be shown that $f$ is $(1,0,d)$-fair over $X$. However, $\calF_\PI$ is not $(\alpha,\beta,d)$-fair on any point pair $(x,x') \in X^2$. The reason is that since $f$ is almost maximally stochastic (i.e. $f(x) \approx 1/2$ for all $x$), and $\calH_\PI$ is pairwise-independent, the binary outputs $\fhat(x)$ and $\fhat(x')$ are about as likely to be the same as they are likely to be different. Hence $\E_{\fhat \sim \calF_\PI} [|\fhat(x) - \fhat(x')|] \approx 1/2$, violating $(\alpha,\beta,d)$-metric fairness. See \Cref{appendix-subsection:unfairness-of-pairwise-independent-derandomization} for the full proof.
\end{proof}

\subsection{Random Threshold Classifier}
\label{subsection:random-threshold-derandomization}
It turns out that there is a near-trivial derandomization that achieves optimal preservation of metric fairness, namely the following \emph{random threshold} classifier family:
\begin{align}
\label{equation:F_RT}
    \calF_\RT := \{\fhat_r ~|~ r \in [0,1]\},
    \text{ where }
    \fhat_r := \Indicator\{f(x) \geq r\}
\end{align}

Formally we make the following observation, whose proof is in \cref{appendix-subsection:random-threshold-derandomization-guarantees}.

\begin{proposition}[Random threshold derandomization guarantees] \label{proposition:random-threshold-derandomization-guarantees}
Let $f$ be an $(\alpha,\beta,d)$-fair stochastic classifier and $\calD$ a distribution over $X$. Then the deterministic classifier family $\calF_\RT$ is also $(\alpha,\beta,d)$-fair. Moreover,
\begin{align*}
    \bias(\fhat_\RT,f,\calD)
        = 0
    \qquad \mathrm{and} \qquad
    \variance(\fhat_\RT,f,\calD)
        \leq \E_{x \sim \calD} [
                f(x)(1-f(x))
            ]
\end{align*}
\end{proposition}

Note that while this derandomization preserves the original fairness parameters perfectly, its variance can be substantially higher than that of $\calF_\PI$ depending on the choice of bucketing function $\pi$ in \Cref{equation:PI-classifier-family}.

One subtlety here is that $\calF_\RT$ is an infinite set, and is therefore not sampleable in practice. For the more realistic scenario in which the threshold $r$ is a number of some fixed precision $\epsilon>0$, the statements in \Cref{proposition:random-threshold-derandomization-guarantees} hold up to additive error $\epsilon$, and $\fhat_\RT$ can be sampled using $O(\log(1/\epsilon))$ uniform random bits. In this case $\calF_\RT$ is $(\alpha,\beta + \epsilon,d)$-fair, and as we can show, this is in fact necessary:

\begin{proposition}[$(\alpha,0,d)$-metric fairness is impossible for finite deterministic families] \label{proposition:fairness-is-impossible-for-finite-deterministic-families}
Let $d: X \times X \to [0,1]$ be a metric over a convex set $X \subseteq \RR^n$, and let $\calF$ be a finite family of deterministic classifiers, at least one of which is nontrivial. Then for every $\alpha \geq 1$ and $\beta < 1/|\calF|$, $\calF$ is not $(\alpha,\beta,d)$-fair.
\end{proposition}

\begin{proof}[Proof sketch]
Since $\calF$ contains a nontrivial classifier $\fhat$, we can pick sufficiently close points around a discontinuity of $\fhat$ and show that in expectation, $\calF$ fails to achieve roughly $(\alpha,1/|\calF|,d)$-fairness on this point pair. See \cref{appendix-subsection:perfect-deterministic-fairness-is-impossible-for-finite-families} for details.
\end{proof}

The main consequence is that there is an irreducible amount of additive unfairness $\beta > 0$ that cannot be avoided when constructing a fair deterministic classifier family. Indeed, the derandomization $\calF$ we present in \cref{section:fair-derandomization-via-locality-sensitive-hashing} has $|\calF| \geq 1/\beta$, thus avoiding the impossible regime indicated by \cref{proposition:fairness-is-impossible-for-finite-deterministic-families}.
\section{Fair Derandomization via Locality-Sensitive Hashing}
\label{section:fair-derandomization-via-locality-sensitive-hashing}
In this section, we construct a deterministic classifier family that combines much of the appeal of both the pairwise-independent derandomization (low output variance) and the random threshold derandomization (strong fairness preservation). This new approach utilizes two types of hashing: first, a pairwise-independent hash family $\calH_\PI$ as before; and second, a locality-sensitive hash family:\footnote{We use the definition of LSH as coined by Charikar \cite{charikar2002similarity}. See \cite{indyk1998approximate} for an alternative gap-based definition in the same spirit.}

\begin{definition}[Locality-sensitive hash (LSH) family]
Let $X$ be a set of hashable items, $B$ a set of buckets, and $d: X^2 \rightarrow[0,1]$ a metric distance function. We say a set $\calH_\LS$ of functions $h: X \rightarrow B$ is a \emph{locality-sensitive family of hash functions} for $d$ if for all $x,x' \in X$,
\begin{align*}
    \Pr_{h \sim \calH_\LS} \left[h(x) \neq h(x')\right] = d(x,x')
\end{align*}
\end{definition}

Locality-sensitive hashing is a well-studied technique, and LSH families have been constructed for many standard distances and similarities, such as $L_1$ \cite{indyk1998approximate}, $L_2$ \cite{andoni2006near}, cosine \cite{charikar2002similarity}, Jaccard \cite{broder1997resemblance}, various data-dependent metrics \cite{jain2008fast,andoni2014beyond,andoni2015optimal}, and more.

Our derandomization works as follows: suppose $f: X \to [0,1]$ is a stochastic classifier, $\calH_\LS$ is a family of locality-sensitive hash functions $h_\LS: X \to B$, and $\calH_\PI$ is a family of pairwise-independent hash functions $h_\PI: B \to [k]$ for some positive integer $k$. Our family of deterministic classifiers is then
\begin{align}
\label{equation:F_LS}
    \calF_\LS
    := \left\{
            \fhat_{h_\LS,h_\PI}
            ~\middle|~
            h_\LS \in \calH_\LS, h_\PI \in \calH_\PI
        \right\},
    \quad \text{where} \quad
    \fhat_{h_\LS,h_\PI}(x)
    := \Indicator\left\{
        f(x) \geq \frac{h_\PI(h_\LS(x))}{k}
    \right\}.
\end{align}

Let us develop some intuition for this construction. First, thinking of $k$ as large, each $\fhat \in \calF_\LS$ essentially assigns a pseudo-random threshold $\frac{h_\PI(h_\LS(x))}{k} \in [0,1]$ to each input $x$, so that $\fhat(x) = 1$ if and only if $f(x)$ exceeds the threshold. Since the outer hash function $h_\PI$ is pairwise-independent, and therefore also uniform, $h_\PI(h_\LS(\cdot))$ is uniform over $[k]$. This endows $\calF_\LS$ with low bias and variance with respect to $f$, as we explain in \cref{subsection:approximation-of-outputs}.

Second, the composition of two different hash functions gives us our fairness guarantee: $h_\LS$ maps close point pairs $x,x'$ to the same bucket, then $h_\PI$ disperses pairs that were not hashed together --- most of which are distant. This separation of point pairs by distance is precisely what enables good preservation of metric fairness, as we prove in \cref{subsection:preservation-of-metric-fairness}.

\subsection{Approximation of Outputs}
\label{subsection:approximation-of-outputs}
We show the following bounds on the bias and variance of our derandomization. The proof is deferred to \cref{appendix-subsection:output-approximation-of-locality-sensitive-derandomization}.

\begin{theorem}[Bias and variance of derandomized classifier]
\label{theorem:LSH-derandomization-bias-variance}
Let $f$ be a stochastic classifier, $\fhat \sim \calF_\LS$, and $\calD$ a distribution over $X$. Then
\begin{align*}
    \bias(\fhat,f,\calD)
        \leq \frac{1}{k}
    \quad \mathrm{and} \quad
    \variance(\fhat,f,\calD)
        \leq \E_{h_\LS \sim \calH_\LS} \left[
                \max_{b \in B} \Pr_{x \sim \calD} [h_\LS(x) = b]
            \right]
            \cdot \E_{x \sim \calD} [f(x)(1-f(x))]
        + \frac{1}{k}.
\end{align*}
\end{theorem}

The above variance bound is similar in form to that of the pairwise-independent derandomization (\cref{theorem:bias-and-variance-of-pairwise-independent-derandomization}), but with added randomization over the sampling of locality-sensitive hash function: when most choices of $h_\LS$ distribute points $x \sim \calD$ into buckets relatively evenly, the bound is as small as $O(1/|B|)$; when most hashes are collisions, the bound may be as large as $\E_{x \sim \calD} [f(x)(1-f(x))]$, essentially tracking the stochasticity of $f$.

\subsection{Preservation of Metric Fairness}
\label{subsection:preservation-of-metric-fairness}
We can now show that our derandomization procedure approximately preserves metric fairness, both in the sense of expected fairness for any pair of points (the usual convention in the metric fairness literature), as well as in aggregate over all point pairs.

\begin{theorem}[Locality-sensitive derandomization preserves metric fairness] \label{theorem:LSH-derandomization-fairness}
Let $f$ be an $(\alpha,\beta,d)$-fair stochastic classifier, where $d$ is a metric with an LSH family $\calH_\LS$ with $k \geq 2/\epsilon$ buckets. Then $\calF_\LS$ is a deterministic classifier family satisfying the following:

\begin{itemize}
    \item (Pairwise fairness) Consider any $x,x' \in X$, and assume without loss of generality that $f(x) \leq f(x')$. Then
    \begin{align*}
        \E_{\fhat \sim \calF_\LS}\left[
                \left|\fhat(x) - \fhat(x')\right|
            \right]
        &\leq [\alpha + 2f(x)(1-f(x'))] \cdot d(x,x')
            + \beta
            + \epsilon
    \end{align*}


    \item (Aggregate fairness)
    For any distance threshold $\tau \in [0,1]$, with probability at least $1-\delta$ over the sampling of $\fhat$,
    \begin{align*}
        \Pr_{(x,x') \sim X^2_\leqtau} \left[
                \fhat(x) \neq \fhat(x')
            \right]
        \leq \left(1 + \frac{1}{\sqrt{\delta}}\right) ([\alpha + 2f(x)(1-f(x'))] \cdot \tau + \beta + \epsilon).
    \end{align*}
\end{itemize}
\end{theorem}


The above fairness guarantees can be simplified by noticing that since $f(x) \leq f(x')$ w.l.o.g., $f(x)(1-f(x')) \leq 1/4$; this yields the following worst-case bounds over $f$ and $(x,x')$:

\begin{corollary}[Worst-case fairness]
\label{corollary:LSH-derandomization-worst-case-fairness}
When $f$ is $(\alpha,\beta,d)$-fair, $\calF_\LS$ satisfies the following:
\begin{itemize}
    \item (Pairwise fairness)
    $\left(\alpha+\frac{1}{2}, \beta+\epsilon, d\right)$-metric fairness on any $(x,x') \in X^2$, i.e.
    \begin{align*}
        \E_{\fhat \sim \calF_\LS} \left[\left|
                \fhat(x) - \fhat(x')
            \right|\right]
        &\leq \left(\alpha + \frac{1}{2}\right) \cdot d(x,x')
            + \beta
            + \epsilon.
    \end{align*}

    \item (Aggregate fairness)
    For any distance threshold $\tau \in [0,1]$, with probability at least $1-\delta$ over the sampling of $\fhat$,
    \begin{align*}
        \Pr_{(x,x') \sim X^2_\leqtau} \left[
                \fhat(x) \neq \fhat(x')
            \right]
        \leq \left(1 + \frac{1}{\sqrt{\delta}}\right)
            \left(
                \alpha \tau
                + \frac{\tau}{2}
                + \beta
                + \epsilon
            \right).
    \end{align*}
\end{itemize}

\end{corollary}

In expectation and with high probability, therefore, the generated deterministic classifier approximates the fairness guarantee of the original classifier to within a small constant factor when there exists an LSH family $\calH$ for $d$. To get a better sense what kind of guarantees this gives us, consider the following example:

\begin{example}
Let $f$ be a $(1,0,d)$-fair stochastic classifier, and suppose we derandomize it to some $\fhat \sim \calF_\LS$, choosing $k = 500$. Then by \cref{corollary:LSH-derandomization-worst-case-fairness},
\begin{itemize}
    \item (Pairwise fairness) $\fhat$ is $(3/2,\epsilon,d)$-metric fair.
    \item (Aggregate fairness) With probability at least $1-\delta = 3/4$ (over the sampling of $\fhat$), at least $76\%$ of point pairs within distance $\tau = 1/20$ receive identical predictions.
\end{itemize}
\end{example}

We present a sketch of the proof of \cref{theorem:LSH-derandomization-fairness}; see \cref{appendix-subsection:fairness-of-LSH-derandomization} for the complete proof.

\begin{proof}[Proof sketch of \cref{theorem:LSH-derandomization-fairness}]
Consider any $x,x' \in X$. Since $\fhat$ is binary and $\calH_\LS$ is locality-sensitive,
\begin{align*}
    \E_{\fhat \sim \calF_\LS} \left[
            \left| \fhat(x) - \fhat(x') \right|
        \right]
    =& \Pr_{\substack{h_\LS \sim \calH_\LS \\ h_\PI \sim \calH_\PI}} \left[
            \fhat(x) \neq \fhat(x')
        \right] \\
    =& \Pr_{h_\LS, h_\PI} \left[
            \fhat(x) \neq \fhat(x')
            ~\middle|~
            h_\LS(x) = h_\LS(x')
        \right] \cdot (1-d(x,x')) \\
        &\qquad
        + \Pr_{h_\LS, h_\PI} \left[
            \fhat(x) \neq \fhat(x')
            ~\middle|~
            h_\LS(x) \neq h_\LS(x')
        \right] \cdot d(x,x')
\end{align*}
From here, the proof is a systematic analysis of conditional probabilities. To give some intuition, notice that the event $[\fhat(x) \neq \fhat(x') ~|~ h_\LS(x) = h_\LS(x')]$ occurs precisely when $\frac{h_\PI(h_\LS(x))}{k}$ falls between $f(x)$ and $f(x')$; by the uniformity of $\calH_\PI$, the probability of this is roughly $|f(x)-f(x')| \leq \alpha \cdot d(x,x') + \beta$. This is one of several cases that use the uniformity and symmetry properties of the composed hash function $h_\PI(h_\LS(\cdot))$ to express $|\fhat(x) - \fhat(x')|$ in terms of $|f(x) - f(x')|$. In some cases this is not possible, resulting in an additive $2f(x)(1-f(x'))$ loss in $\alpha$.
\end{proof}

\subsection{Sample Complexity}
Since the LSH-based derandomization procedure involves sampling two hash functions $\calH_\PI$ and $\calH_\LS$, it samples $\fhat$ using $O(\log|B| + \log k + S_d(X, B))$ random bits, where $O(\log|B| + \log k)$ is the number of bits used to sample a pairwise-independent hash function \cite{pairwise-independent-hashing-notes}, and $S_d(X, B)$ is the number of random bits required to sample a locality-sensitive hash function for metric $d$ with domain $X$ and range $B$. When the metric is the Euclidean distance, for example, $O(\dim X)$ random bits suffice \cite{lsh-euclidean-distance-notes}.





\section{Structural Lemmas for Fair Classifier Derandomization}
\label{section:structural-lemmas}
In this section, we present generic results applicable to all classifier derandomization procedures, as well as unify different definitions of fairness used in this paper and others.

\subsection{Bias-Variance Decomposition}
Up to this point, a ``stochastic'' classifier has signified any function $f$ from $X$ to $[0,1]$; in this sense, it does not necessarily contain any randomness of its own. However, when it comes time to perform a binary decision on some input $x$, $f(x)$ is typically interpreted as the probability of outputting $1$, i.e. we use the (truly random) binary function $\indicatorf(x) \sim \mathrm{Bern}(f(x))$.

By how much does this prediction typically differ from that of some pre-sampled deterministic classifier $\fhat$? We show that this error can be decomposed into the bias of $\fhat$ and the variance of both $\fhat$ and $f$:

\begin{lemma}[Bias-variance decomposition]
\label{lemma:bias-variance-decomposition}
Let $f: X \to [0,1]$ be a stochastic classifier and $\calF$ a deterministic classifier family. Then for any $x \in X$,
\begin{align*}
    \E_{f,\fhat} \left[\left|\fhat(x) - \indicatorf(x)\right|\right]
    \leq \left| \bias(\fhat,\indicatorf,x) \right|
        + 2 \left(
            \Var_f (\indicatorf(x))
            + \Var_{\fhat \sim \calF} \left(\fhat(x)\right)
        \right)^{2/3}
\end{align*}
\end{lemma}

We defer the proof to \cref{appendix-subsection:bias-variance-decomposition}. For now, let us interpret this decomposition and see how it applies to the derandomization approaches laid out in previous sections. Recall that for all three derandomizations --- $\calF_\PI$, $\calF_\RT$, and $\calF_\LS$ --- the bias was either zero or could be made arbitrarily small. As for the variance, we see two types: the first, $\Var_f (\indicatorf(x))$, is equal to $f(x)(1-f(x))$, i.e. the variance of a Bernoulli with parameter $f(x)$; it therefore quantifies the inherent stochasticity of the given classifier $f$, over which we have no control. In contrast, the second variance arises from sampling the deterministic classifier $\fhat$, which depends greatly on the procedure being used. Thus a comparison of the expected error of these approaches boils down to this latter variance, for which the pairwise-independent and locality-sensitive hashing approaches compare favorably against the simple random threshold.

\subsection{Metric Fairness and Threshold Fairness}
Friedler, Scheidegger, and Venkatasubramanian \cite{friedler2016impossibility} propose an alternative threshold-based notion of individual fairness that implements the mantra that ``similar individuals should receive similar treatment,'' but only extends this constraint to pairs of inputs within a certain distance of interest:

\begin{definition}[$(\sigma,\tau,d)$-threshold fairness] \label{definition:threshold-fairness}
Fix some constants $\sigma, \tau \in (0,1)$. We say a stochastic classifier $f$ is \emph{$(\sigma,\tau,d)$-threshold fair} if for all $x,x' \in X$ such that $d(x,x') \leq \sigma$, we have $|f(x)-f(x')| \leq \tau$. We say a deterministic classifier family $\calF$ is $(\sigma,\tau,d)$-threshold fair if for all $x,x' \in X$ such that $d(x,x') \leq \sigma$, we have $\E_{\fhat \sim \calF}[|\fhat(x) - \fhat(x')|] \leq \tau$.
\end{definition}

Neither metric fairness nor threshold fairness fully subsumes the other. However, we can still show the following \emph{algorithmic} reduction: if we wish to derandomize a stochastic classifier while preserving threshold fairness, then it suffices to use any procedure that preserves metric fairness. For example, suppose we have a derandomization procedure that worsens the input classifier's fairness parameters $\alpha$ and $\beta$ to at most $a \cdot \alpha$ and $b \cdot \beta$, respectively, for some small constants $a, b \geq 1$. We should also expect this procedure to preserve threshold fairness, within certain parameters related to $a,b$. This is what we prove in the following lemma, but for more general fairness preservation functions:

\begin{lemma}[Metric-fair derandomization preserves threshold fairness]
\label{lemma:metric-fair-derandomization-preserves-threshold-fairness}
Suppose we have a procedure that, given an $(\alpha,\beta,d)$-metric fair stochastic classifier $f$, samples a deterministic classifier $\fhat$ from an $(A(\alpha), B(\beta), d)$-metric fair family $\calF$, for some functions $A,B : \RR \to \RR$. Then this same procedure also derandomizes any $(\sigma,\tau,d)$-threshold fair stochastic classifier to a deterministic classifier from a $(\sigma, A(0) \cdot \sigma + B(\tau), d)$-threshold fair family.
\end{lemma}

Applying this to the random threshold and locality-sensitive derandomization procedures yields the following:

\begin{corollary}[Threshold fairness-preserving derandomizations]
\label{corollary:threshold-fairness-preserving-derandomizations}
Let $f$ be a $(\sigma,\tau,d)$-threshold fair stochastic classifier. Then
\begin{itemize}
    \item The family $\calF_\RT$ is $(\sigma,\tau,d)$-threshold fair.
    \item If $d$ is LSHable, the family $\calF_\LS$, for a choice of $k \geq 4/\sigma$, is $(\sigma, \sigma+\tau,d)$-threshold fair.
\end{itemize}
\end{corollary}

The proofs are deferred to \cref{appendix-subsection:metric-fair-derandomization-preserves-threshold-fairness}.

\subsection{Pairwise Fairness and Aggregate Fairness}
Throughout most of this paper (and in most of the individual fairness literature), we have been focused on pairwise notion of fairness, such as metric fairness (\cref{definition:metric-fairness}) and threshold fairness (\cref{definition:threshold-fairness}). One shortcoming of these definitions is that even if a classifier satisfies them for any particular pair of points $(x,x')$, they do not hold simultaneously for all input pairs; thus once we sample a specific deterministic classifier $\fhat$, it may be unfair for many pairs. Fortunately, as we now show, these pairwise statements imply high-probability aggregate fairness guarantees: if $\calF$ is a metric-fair family, then \emph{most} deterministic classifiers in $\calF$ assign \emph{most} close pairs the same prediction.

To that end, for all distances $\tau \in [0,1]$, let $X^2_\leqtau := \left\{(x,x') \in X^2 ~\middle|~ d(x,x') \leq \tau\right\}$ denote the set of point pairs within distance $\tau$. Then we can bound the fraction of $\tau$-close pairs that receive different predictions:

\begin{lemma}[Pairwise fairness implies aggregate fairness]
\label{lemma:pairwise-fairness-implies-aggregate-fairness}
Let $\calF$ be an $(\alpha,\beta,d)$-fair deterministic classifier family. Then for any distance threshold $\tau \in [0,1]$, with probability at least $1-\delta$ over the sampling of $\fhat \sim \calF$,
\begin{align*}
    \Pr_{(x,x') \sim X^2_\leqtau} \left[
            \fhat(x) \neq \fhat(x')
        \right]
    \leq \left(1 + \frac{1}{\sqrt{\delta}}\right) (\alpha\tau + \beta).
\end{align*}
\end{lemma}

The proof is deferred to \cref{appendix-subsection:pairwise-fairness-implies-aggregate-fairness}.

\subsection{Output Approximation and Loss Approximation}
\label{subsection:output-approximation-and-loss-approximation}
In this paper, we have analyzed the output approximation qualities of various derandomization techniques using the definitions of bias and variance in \cref{subsection:preliminaries}, which say that the output of $\fhat$ should resemble that of $f$, either on a single point $x$ or in aggregate over some distribution $\calD$.

An alternative set of definitions of bias and variance, put forth in \cite{cotter2019making}, instead measures how well $\fhat$ preserves the \emph{loss} of $f$ according to one or more binary loss functions $\ell$. This property, which we might call \emph{loss approximation}, is useful since in practice, classifiers are typically compared based on criteria such as accuracy, false positive rate, etc. evaluated on a dataset --- and these are essentially binary loss functions averaged over a data distribution.

Concretely, let $\ell: \{0,1\} \times \{0,1\} \to \{0,1\}$ be a loss function and let $(x,y) \in X \times \{0,1\}$ be an instance with its corresponding label. The loss on this instance incurred by a (stochastic or deterministic) classifier $f$ is defined as
\begin{align*}
    L(f,x,y)
    := f(x)\ell(1,y) + (1-f(x))\ell(0,y)
\end{align*}
The (pointwise) bias and variance of $\fhat$ under this loss are then
\begin{align*}
    \bias(\fhat,f,x,y,\ell)
        := \left|
            \E_{\fhat \sim \calF} \left[
                L(\fhat,x,y)
            \right]
            - L(f,x,y)
        \right|
    \qquad \mathrm{and} \qquad
    \variance(\fhat,x,y,\ell)
        := \Var_{\fhat \sim \calF} \left(L(\fhat,x,y)\right)
\end{align*}

We observe that these are closely related to the simpler definitions given in \cref{subsection:preliminaries}:

\begin{lemma}
\label{lemma:output-approximation-implies-loss-approximation}
For any $\ell: \{0,1\} \times \{0,1\} \to \{0,1\}$, $x \in X$, and $y \in \{0,1\}$,
\begin{align*}
    \bias(\fhat,f,x,y,\ell)
        \leq \left|\bias(\fhat,f,x)\right|
    \qquad \mathrm{and} \qquad
    \variance(\fhat,x,y,\ell)
        \leq \variance(\fhat,x)
\end{align*}
\end{lemma}

Thus even when the goal is to compute a derandomization that simulates the performance of $f$ on one or more binary loss functions, it essentially suffices to use a derandomization that merely simulates the raw output of $f$ itself. See \cref{appendix-subsection:output-approximation-and-loss-approximation} for the proof of this lemma.
\section{Discussion}
\label{section:discussion}
We offer some brief notes regarding practical considerations for our derandomization framework.

\paragraph{A framework for derandomization}
Our results give machine learning practitioners a time- and space-efficient way to remove randomness --- with the inherent brittleness, security vulnerabilities, and other issues that stochasticity entails --- from their deployed models while approximately preserving fairness constraints enforced during training. Notably, our derandomization procedure has the useful quality of being \emph{oblivious} to $f$, its training process, and even its actual fairness parameters $\alpha$ and $\beta$. It can therefore be applied as an independent post-processing step --- for example, on the stochastic classifiers generated by the algorithms of \cite{rothblum2018probably}, \cite{kim2018fairness}, and others. The burden on the model designer is thus reduced to selecting a metric feature space $(X,d)$ that is both appropriate for the classification task and for which an LSH family exists.

This simplification comes with inherent constraints: it was shown in \cite{charikar2002similarity} that only metrics (or similarities $\phi$ whose complement $d$ is a metric) can have LSH schemes, though not all of them do. On the positive side, recent work has shown that various non-LSHable similarities can be approximated by LSHable similarities with some provable distortion bound \cite{chierichetti2019distortion}.

\paragraph{Separation of feature sets}
Throughout this paper, we have assumed that the inner hash function $h_\LS$ and classifiers $f$ and $\fhat$ all share the same domain $X$; however, this is in no way necessary. In fact, from a fairness perspective, it is often prudent to distinguish between the features used for ensuring fairness and those used purely for inference, i.e. we may have
\begin{align*}
    f: X \to [0,1], \
    \fhat: X \to \{0,1\},
    \text{ and }
    h_\LS: Z \to B
\end{align*}
The feature set $Z$ should be chosen, in tandem with an appropriate LSHable metric $d: Z \to [0,1]$, so as to measure similarity or difference between inputs on the basis of attributes that should be treated equitably; on the other hand, the feature set $X$ can be designed primarily to maximize predictive accuracy, and need not have any overlap with $Z$. The fairness guarantees of \cref{theorem:LSH-derandomization-fairness} and \cref{corollary:LSH-derandomization-worst-case-fairness} then hold with respect to the metric space $(Z,d)$ rather than $(X,d)$.

\paragraph{Future work: guarantees for protected attributes}
This paper has focused on classifier derandomization with individual fairness guarantees, but it is also worthwhile to investigate the effect of derandomization from a group fairness perspective --- for example, if it is possible to design an LSHable metric such that the derandomization preserves notions of fairness with respect to a protected attribute.

\paragraph{Acknowledgement}
This work is partially supported by the National Science Foundation (NSF) under grants IIS-2007951, IIS-2143895, IIS-2040800 (FAI program in collaboration with Amazon), and CCF-2023495.

\printbibliography

\newpage
\appendix
\section{Omitted Proofs}

\subsection{Unfairness of Pairwise-Independent Derandomization}
\label{appendix-subsection:unfairness-of-pairwise-independent-derandomization}
\begin{proof}[Proof of \cref{proposition:unfairness-of-pairwise-independent-derandomization}]
For any $\delta > 0$, let $\SS_\delta := \{x \in \RR^n ~|~ d(x,\mathbf{0}) = \delta\}$ be the sphere of radius $\delta$ around the origin. Consider any $\alpha \geq 1$ and $\beta \in \left(0,\frac{1}{2}-\frac{1}{2k}\right)$, and choose $X$ to be some subset of $\SS_\delta$ of size $|X| = N$ in which the closest two points are positioned at distance $\epsilon$ from one another, where
\begin{align*}
    0 < \epsilon := \min_{x,x' \in X} d(x,x') < \frac{1}{2} - \frac{1}{2k} - \beta.
\end{align*}

Now let $f$ be a classifier that maps half of the points in $X$ to $\frac{1 + \epsilon}{2}$, and the other half to $\frac{1 - \epsilon}{2}$. $f$ is $(1,0,d)$-fair over $X$, since for any $x,x' \in X$,
\begin{align*}
    |f(x) - f(x')|
    \leq \left|\frac{1 + \epsilon}{2} - \frac{1 - \epsilon}{2}\right|
    = \epsilon
    \leq d(x,x')
\end{align*}

However, $\calF_\PI$ is not $(\alpha,\beta,d)$-fair on any point pair. To see this, consider any $x \neq x' \in X$; we show that for $\fhat \sim \calF_\PI$, $|\fhat(x) - \fhat(x')|$ is typically large relative to $d(x,x')$:
\begin{align*}
    \E_{\fhat \sim \calF_\PI} \left[
        \left| \fhat(x) - \fhat(x') \right|
    \right]
    &= \Pr_{\fhat \sim \calF_\PI} \left[\fhat(x) \neq \fhat(x')\right]
        && \tag{$\fhat \in \{0,1\}$} \\
    &= \Pr_{\fhat \sim \calF_\PI} \left[
            \fhat(x) = 1,
            \fhat(x') = 0
        \right]
        + \Pr_{\fhat \sim \calF_\PI} \left[
            \fhat(x) = 0,
            \fhat(x') = 1
        \right] \\
    &= \Pr_{h \sim \calH_\PI} \left[
            f(x) \geq \frac{h(x)}{k},
            f(x') < \frac{h(x')}{k}
        \right]
        + \Pr_{h \sim \calH_\PI} \left[
            f(x) < \frac{h(x)}{k},
            f(x') \geq \frac{h(x')}{k}
        \right] \\
    &\geq \Pr_{h \sim \calH_\PI} \left[
            \frac{1-\epsilon}{2} \geq \frac{h(x)}{k},
            \frac{1+\epsilon}{2} < \frac{h(x')}{k}
        \right]
        + \Pr_{h \sim \calH_\PI} \left[
            \frac{1+\epsilon}{2} < \frac{h(x)}{k},
            \frac{1-\epsilon}{2} \geq \frac{h(x')}{k}
        \right] \\
    &= \Pr_{h \sim \calH_\PI} \left[
            \frac{h(x)}{k} \leq \frac{1 - \epsilon}{2}
        \right]
        \cdot \Pr_{h \sim \calH_\PI} \left[
            \frac{h(x')}{k} > \frac{1 + \epsilon}{2}
        \right]\\
    & \qquad    + \Pr_{h \sim \calH_\PI} \left[
            \frac{h(x)}{k} > \frac{1 + \epsilon}{2}
        \right]
        \cdot \Pr_{h \sim \calH_\PI} \left[
            \frac{h(x')}{k} \leq \frac{1 - \epsilon}{2}
        \right]
        \tag{by pairwise independence} \\
    &\geq \left(\frac{1 - \epsilon}{2} - \frac{1}{k}\right) \left(1 - \frac{1 + \epsilon}{2} - \frac{1}{k}\right)
        + \left(1 - \frac{1 + \epsilon}{2} - \frac{1}{k}\right) \left(\frac{1 - \epsilon}{2} - \frac{1}{k}\right)
        \tag{by \eqref{equation:hash-interval-bounds}} \\
    &= \frac{1}{2} \left(1 - 2\epsilon + \epsilon^2\right) - \frac{1 - \epsilon}{2k} + \frac{1}{k^2} \\
    &\geq \frac{1}{2} - \epsilon - \frac{1}{2k}
\end{align*}

The distance between any two points in $\SS_\delta$, and therefore $X$, is at most $2\delta$; hence for a choice of $\delta \in \left(0, \frac{1/2 - \beta - \epsilon - 1/2k}{2\alpha}\right)$ (which is possible since $\beta < \frac{1}{2}-\frac{1}{2k}$ and $\epsilon < \frac{1}{2} - \frac{1}{2k} - \beta$), we have
\begin{align*}
    \E_{h \sim \calH} \left[
        \left| \fhat(x) - \fhat(x') \right|
    \right]
    \geq \frac{1}{2} - \epsilon - \frac{1}{2k}
    = 2\alpha \cdot \frac{1/2 - \beta - \epsilon - 1/2k}{2\alpha} + \beta
    > \alpha \cdot 2\delta + \beta
    \geq \alpha \cdot d(x,x') + \beta
\end{align*}
which is a violation of $(\alpha,\beta,d)$-metric fairness (\cref{equation:deterministic-individual-fairness-condition}) and applies to all pairs $x,x' \in X$.
\end{proof}

\subsection{Random Threshold Derandomization Guarantees}
\label{appendix-subsection:random-threshold-derandomization-guarantees}
\begin{proof}[Proof of \Cref{proposition:random-threshold-derandomization-guarantees}]
Let $f$ be an $(\alpha,\beta,d)$-fair classifier, and consider any $x,x' \in X$. We have
\begin{align*}
    \E_{\fhat_r \sim \calF_\RT} \left[
            \left| \fhat_r(x) - \fhat_r(x') \right|
        \right]
    &= \Pr_{\fhat_r \sim \calF_\RT} \left[
            \fhat_r(x) \neq \fhat_r(x')
        \right]
        \tag{$\fhat \in \{0,1\}$} \\
    &= \Pr_{\fhat_r \sim \calF_\RT} \left[
            \fhat_r(x) = 0,
            \fhat_r(x') = 1
        \right]
        + \Pr_{\fhat_r \sim \calF_\RT} \left[
            \fhat_r(x) = 1,
            \fhat_r(x') = 0
        \right] \\
    &= \Pr_{r \sim [0,1]} [f(x) < r \leq f(x')]
        + \Pr_{r \sim [0,1]} [f(x') < r \leq f(x)] \\
    &= |f(x) - f(x')| \\
    &\leq \alpha \cdot d(x,x') + \beta
        \tag{$f$ is $(\alpha,\beta,d)$-fair}
\end{align*}
which shows that $\calF_\RT$ is also $(\alpha,\beta,d)$-fair. To compute the bias, note that for any $x \in X$,
\begin{align}
\label{equation:random-threshold-bias}
    \E_{\fhat_r \sim \calF_\RT} \left[
            \fhat_r(x)
        \right]
    = \Pr_{r \sim [0,1]} [f(x) \geq r]
    = f(x)
\end{align}
which implies $\bias(\fhat_r,f,x) = 0$ for all $x$ and hence $\bias(\fhat,f,\calD)$ for all $\calD$. Finally for the variance, we have
\begin{align*}
    \variance(\fhat_r,\calD)
    &:= \Var_{\fhat_r \sim \calF_\RT} \left(
            \E_{x \sim \calD} [\fhat_r(x)]
        \right) \\
    &= \E_{r \sim [0,1]} \left[
            \left(
                \E_{x \sim \calD} \left[
                    \fhat_r(x)
                \right]
            \right)^2
        \right]
        - \left(
            \E_{r \sim [0,1]} \left[
                \E_{x \sim \calD} \left[
                    \fhat_r(x)
                \right]
            \right]
        \right)^2 \\
    &= \E_{r \sim [0,1]} \left[
            \left(
                \E_{x \sim \calD} \left[
                    \fhat_r(x)
                \right]
            \right)^2
        \right]
        - \left(
            \E_{x \sim \calD} \left[
                \E_{r \sim [0,1]} \left[
                    \fhat_r(x)
                \right]
            \right]
        \right)^2 \\
    &= \E_{r \sim [0,1]} \left[
            \E_{x,x' \sim \calD} \left[
                \fhat_r(x)
                \fhat_r(x')
            \right]
        \right]
        - \E_{x,x' \sim \calD} \left[
            \E_{r \sim [0,1]} \left[
                \fhat_r(x)
            \right]
            \E_{r \sim [0,1]} \left[
                \fhat_r(x')
            \right]
        \right] \\
    &= \E_{x,x' \sim \calD} \left[
            \E_{r \sim [0,1]} \left[
                \fhat_r(x)
                \fhat_r(x')
            \right]
            - \E_{r \sim [0,1]} \left[
                    \fhat_r(x)
                \right]
                \E_{r \sim [0,1]} \left[
                    \fhat_r(x')
                \right]
        \right] \\
    &= \E_{x,x' \sim \calD} \left[
            \Cov_{r \sim [0,1]} \left(
                \fhat_r(x),
                \fhat_r(x')
            \right)
        \right] \\
    &\leq \E_{x,x' \sim \calD} \left[
            \sqrt{
                \Var_{r \sim [0,1]} \left(
                    \fhat_r(x)
                \right)
                \Var_{r \sim [0,1]} \left(
                    \fhat_r(x')
                \right)
            }
        \right]
        \tag{Cauchy-Schwarz inequality} \\
    &= \left(
            \E_{x \sim \calD} \left[
                \sqrt{
                    \Var_{r \sim [0,1]} \left(
                        \fhat_r(x)
                    \right)
                }
            \right]
        \right)^2 \\
    &\leq \E_{x \sim \calD} \left[
            \Var_{r \sim [0,1]} \left(
                \fhat_r(x)
            \right)
        \right]
        \tag{Jensen's inequality} \\
    &= \E_{x \sim \calD} \left[
            \E_{r \sim [0,1]} \left[
                \hat{f}_r(x)
            \right]
            \left(
                1 - \E_{r \sim [0,1]} \left[
                    \hat{f}_r(x)
                \right]
            \right)
        \right] \\
    &= \E_{x \sim \calD} [f(x)(1 - f(x))]
        \tag{\cref{equation:random-threshold-bias}}
\end{align*}
as required.
\end{proof}

\subsection{Perfect Deterministic Fairness is Impossible for Finite Families}
\label{appendix-subsection:perfect-deterministic-fairness-is-impossible-for-finite-families}
\begin{proof}[Proof of \cref{proposition:fairness-is-impossible-for-finite-deterministic-families}]
Consider any $\alpha \geq 1$ and $\beta \in (0, 1/|\calF|)$; it suffices to exhibit a pair of points $x,x' \in X$ such that
\begin{align*}
    \E_{\fhat \sim \calF}\left[ \left|\fhat(x) - \fhat(x')\right| \right]
    > \alpha \cdot d(x,x') + \beta .
\end{align*}

For any $\delta > 0$, define the \emph{ball of radius $\delta$ around $x$} to be $\BB_\delta(x) := \{x' \in X ~|~ d(x,x') \leq \delta\}$. By assumption, $\calF$ contains at least one nontrivial classifier (i.e. one function that is not identically $1$ or $0$); let $\fhat$ be one such classifier. Since $X \subseteq \mathbb{R}^n$ is convex and $d$ is a metric, $\fhat$ must be discontinuous at some point $x \in X$, meaning that for all $\delta > 0$, there exists $x' \in \BB_\delta(x)$ such that $\fhat(x) = 1-\fhat(x')$. Choose any $\delta^* \in \left(0, \frac{1/|\calF| - \beta}{\alpha}\right)$, and consider some $x^* \in \BB_{\delta^*}(x)$. We have
\begin{align*}
    \E_{\fhat \sim \calF}\left[ \left|\fhat(x) - \fhat(x^*)\right| \right]
    &\geq \frac{1}{|\calF|}
        \tag{at least one function in $\calF$ is discontinuous at $x$} \\
    &= \alpha \left(\frac{1/|\calF| - \beta}{\alpha}\right) + \beta \\
    &> \alpha \cdot \delta^* + \beta
        \tag{$\delta^* < \frac{1/|\calF| - \beta}{\alpha}$} \\
    &\geq \alpha \cdot d(x,x^*) + \beta
        \tag{$x^* \in \BB_{\delta^*}(x)$}
\end{align*}
which shows that $\calF$ is not $(\alpha,\beta,d)$-fair.
\end{proof}

\subsection{Output Approximation of Locality-Sensitive Derandomization}
\label{appendix-subsection:output-approximation-of-locality-sensitive-derandomization}
\begin{proof}[Proof of \Cref{theorem:LSH-derandomization-bias-variance}]
We will repeatedly use the following fact: by the uniformity of $\calH_\PI$, for all $0 \leq a < b \leq 1$ and $x \in X$ we have
\begin{align} \label{equation:hash-interval-bounds}
    \Pr_{\substack{h_\LS \sim \calH_\LS \\ h_\PI \sim \calH_\PI}} \left[
            a \leq \frac{h_\PI(h_\LS(x))}{k} \leq b
        \right]
    \in \left(b - a - \frac{1}{k}, b - a + \frac{1}{k}\right)
\end{align}

Thus for all $x \in X$,
\begin{align*}
    \E_{\fhat \sim \calF_\LS} \left[\fhat(x)\right]
    = \Pr_{\fhat \sim \calF_\LS} \left[\fhat(x) = 1\right]
    = \Pr_{\substack{h_\LS \sim \calH_\LS \\ h_\PI \sim \calH_\PI}} \left[
            f(x) \geq \frac{h_\PI(h_\LS(x))}{k}
        \right]
    \in \left(
        f(x) - \frac{1}{k},
        f(x) + \frac{1}{k}
    \right)
\end{align*}
which implies $\bias(\fhat,f,x) \leq \frac{1}{k}$ for all $x \in X$ and hence $\bias(\fhat,f,\calD) \leq \frac{1}{k}$ for all $\calD$.

Now we bound the variance. Define the \emph{bucketed} stochastic classifier
\begin{align*}
    g(x) = \frac{1}{k} \sum_{i=1}^k \Indicator \left\{ f(x) \geq \frac{i}{k} \right\}
\end{align*}
In other words, $g(x)$ is the smallest multiple of $1/k$ greater than $f(x)$. Note that $|g(x)-f(x)| \leq \frac{1}{k}$ for all $x$. Additionally, define the \emph{deterministic} classifier family $\calG_\LS$ from $g$ just as $\calF_\LS$ was defined from $f$ in \cref{equation:F_LS}, i.e.
\begin{align}
    \calG_\LS
    := \left\{
            \hat{g}_{h_\LS,h_\PI}
            ~\middle|~
            h_\LS \in \calH_\LS,
            h_\PI \in \calH_\PI
        \right\},
    \quad \text{where} \quad
    \hat{g}_{h_\LS,h_\PI}(x)
    := \Indicator\left\{
            g(x) \geq \frac{h_\PI(h_\LS(x))}{k}
        \right\}.
\end{align}

It essentially suffices to analyze $\hat{g}$ instead of $\fhat$, since in the end, we simply incur an additional bias or variance of $\frac{1}{k}$. To begin, observe that for any distribution $\calD$ over $X$,
\begin{align*}
    \variance(\fhat,f,\calD)
    &= \variance(\hat{g},g,\calD) \\
    &:= \Var_{\hat{g} \sim \calG_\LS} \left(
            \E_{x \sim \calD} [\hat{g}(x)]
        \right) \\
    &= \E_{\substack{h_\LS \sim \calH_\LS \\ h_\PI \sim \calH_\PI}} \left[
                \left(
                    \E_{x \sim \calD} [\hat{g}(x)]
                \right)^2
            \right]
        - \left(
                \E_{x \sim \calD} [\hat{g}(x)]
            \right)^2 \\
    &= \E_{\substack{h_\LS \sim \calH_\LS \\ h_\PI \sim \calH_\PI}} \left[
                \left(
                    \E_{x \sim \calD} [\hat{g}(x)]
                \right)^2
            \right]
        - \left(
                \E_{x \sim \calD} [g(x)]
            \right)^2
\end{align*}

To evaluate the first term, note that for any $x,x' \in X$,
\begin{align*}
    & \E_{\substack{h_\LS \sim \calH_\LS \\ h_\PI \sim \calH_\PI}} \left[
            \hat{g}(x) \hat{g}(x')
        \right] \\
    &= \E_{h_\LS \sim \calH_\LS} \left[
            \E_{h_\PI \sim \calH_\PI} \left[
                \Indicator\{h_\LS(x) = h_\LS(x')\}
                \hat{g}(x) \hat{g}(x')
            \right]
            + \E_{h_\PI \sim \calH_\PI} \left[
                \Indicator\{h_\LS(x) \neq h_\LS(x')\}
                \hat{g}(x) \hat{g}(x')
            \right]
        \right] \\
    &= \E_{h_\LS \sim \calH_\LS} \left[
            \E_{h_\PI \sim \calH_\PI} \left[
                \Indicator\{h_\LS(x) = h_\LS(x')\}
                \hat{g}(x) \hat{g}(x')
            \right]
            + \Indicator\{h_\LS(x) \neq h_\LS(x')\}
            g(x) g(x')
        \right]
        \tag{pairwise independence}
\end{align*}

Thus the first term of the variance is
\begin{align*}
    \E_{\substack{h_\LS \sim \calH_\LS \\ h_\PI \sim \calH_\PI}} \left[
            \left(
                \E_{x \sim \calD} [\hat{g}(x)]
            \right)^2
        \right]
    &= \E_{\substack{h_\LS \sim \calH_\LS \\ h_\PI \sim \calH_\PI}} \left[
            \E_{x,x' \sim \calD} [
                \hat{g}(x) \hat{g}(x')
            ]
        \right] \\
    &= \E_{x,x' \sim \calD} \left[
            \E_{\substack{h_\LS \sim \calH_\LS \\ h_\PI \sim \calH_\PI}} [
                \hat{g}(x) \hat{g}(x')
            ]
        \right] \\
    &= \E_{x,x' \sim \calD} \left[
            \E_{h_\LS \sim \calH_\LS} \left[
                \E_{h_\PI \sim \calH_\PI} \left[
                    \Indicator\{h_\LS(x) = h_\LS(x')\}
                    \hat{g}(x) \hat{g}(x')
                \right]
                + \Indicator\{h_\LS(x) \neq h_\LS(x')\}
                g(x) g(x')
            \right]
        \right]
\end{align*}
Next consider the second term:
\begin{align*}
    \left(
            \E_{x \sim \calD} [g(x)]
        \right)^2
    = \E_{x,x' \sim \calD} [
            g(x)
            g(x')
        ]
\end{align*}

Putting these together, we have
\begin{align*}
    & \variance(\fhat,f,\calD) \\
    &= \E_{h_\LS \sim \calH_\LS} \left[
        \E_{h_\PI \sim \calH_\PI} \left[
                \E_{x,x' \sim \calD} [
                    \Indicator\{h_\LS(x) = h_\LS(x')\}
                    \hat{g}(x) \hat{g}(x')
                ]
            \right]
            - \E_{x,x' \sim \calD} \left[
                \Indicator\{h_\LS(x) = h_\LS(x')\}
                g(x) g(x')
            \right]
        \right] \\
    &= \E_{h_\LS \sim \calH_\LS} \left[
        \E_{x,x' \sim \calD} \left[
                \Indicator\{h_\LS(x) = h_\LS(x')\} \cdot \left(
                    \E_{h_\PI} [
                        \hat{g}(x) \hat{g}(x')
                    ]
                    - g(x) g(x')
                \right)
            \right]
        \right] \\
    &= \E_{h_\LS \sim \calH_\LS} \left[
        \E_{x,x' \sim \calD} \left[
                \Indicator\{h_\LS(x) = h_\LS(x')\} \cdot \left(
                    \E_{h_\PI} [
                        \hat{g}(x) \hat{g}(x')
                    ]
                    - \E_{h_\PI} [\hat{g}(x)]
                    \E_{h_\PI} [\hat{g}(x')]
                \right)
            \right]
        \right] \\
    &= \E_{h_\LS \sim \calH_\LS} \left[
        \E_{x,x' \sim \calD} \left[
                \Indicator\{h_\LS(x) = h_\LS(x')\} \cdot
                \Cov_{h_\PI} \left(
                    \hat{g}(x),
                    \hat{g}(x')
                \right)
            \right]
        \right] \\
    &\leq \E_{h_\LS \sim \calH_\LS} \left[
        \E_{x,x' \sim \calD} \left[
                \Indicator\{h_\LS(x) = h_\LS(x')\} \cdot
                \sqrt{
                    \Var_{h_\PI} \left(\hat{g}(x)\right)
                    \Var_{h_\PI} \left(\hat{g}(x')\right)
                }
            \right]
        \right]
        \tag{Cauchy-Schwarz inequality} \\
    &= \E_{h_\LS \sim \calH_\LS} \left[
        \sum_{b \in B} \left(
                \E_{x \sim \calD} \left[
                    \Indicator\{h_\LS(x) = b\} \cdot
                    \sqrt{
                        \Var_{h_\PI} \left(\hat{g}(x)\right)
                    }
                \right]
            \right)^2
        \right] \\
    &= \E_{h_\LS \sim \calH_\LS} \left[
            \sum_{b \in B} \left(
                \Pr_{x \sim \calD} [h_\LS(x) = b]
                \cdot
                \E_{x \sim \calD} \left[
                    \sqrt{
                        \Var_{h_\PI} \left(\hat{g}(x)\right)
                    }
                    ~ \middle| ~
                    h_\LS(x) = b
                \right]
            \right)^2
        \right] \\
    &\leq \E_{h_\LS \sim \calH_\LS} \left[
        \sum_{b \in B}
            \left(
                \Pr_{x \sim \calD} [h_\LS(x) = b]
            \right)^2
            \cdot
            \E_{x \sim \calD} \left[
                \Var_{h_\PI} \left(\hat{g}(x)\right)
                ~ \middle| ~
                h_\LS(x) = b
            \right]
        \right]
        \tag{Jensen's inequality} \\
    &= \E_{h_\LS \sim \calH_\LS} \left[
        \sum_{b \in B}
            \left(
                \Pr_{x \sim \calD} [h_\LS(x) = b]
            \right)^2
            \cdot
            \E_{x \sim \calD} [
                g(x)(1-g(x))
                ~|~
                h_\LS(x) = b
            ]
        \right] \\
    &\leq \E_{h_\LS \sim \calH_\LS} \left[
            \left(
                    \max_{b \in B} \Pr_{x \sim \calD} [h_\LS(x)=b]
                \right)
            \sum_{b \in B}
            \Pr_{x \sim \calD} [h_\LS(x)=b]
            \cdot \E_{x \sim \calD} [
                g(x)(1-g(x))
                ~|~
                h_\LS(x) = b
            ]
        \right] \\
    &= \E_{h_\LS \sim \calH_\LS} \left[
            \max_{b \in B} \Pr_{x \sim \calD} [h_\LS(x)=b]
        \right]
        \cdot \E_{x \sim \calD} [
                g(x)(1-g(x))
            ] \\
    &\leq \E_{h_\LS \sim \calH_\LS} \left[
            \max_{b \in B} \Pr_{x \sim \calD} [h_\LS(x)=b]
        \right]
        \cdot \E_{x \sim \calD} \left[
                f(x)(1-f(x))
                + \frac{1}{k}
            \right]
        \tag{$\bias(f,g,x) \leq \frac{1}{k}$ for all $x$}
\end{align*}
\end{proof}

\subsection{Fairness of LSH-Based Derandomization}
\label{appendix-subsection:fairness-of-LSH-derandomization}
\begin{proof}[Proof of \cref{theorem:LSH-derandomization-fairness}]
We first prove pairwise metric fairness. Consider any $x,x' \in X$, and assume without loss of generality that $f(x) \leq f(x')$. We have
\begin{align}
    &\quad \ \ \E_{\fhat \sim \calF_\LS} \left[
            \left| \fhat(x) - \fhat(x') \right|
        \right]
        \nonumber \\
    &= \Pr_{\substack{h_\LS \sim \calH_\LS \\ h_\PI \sim \calH_\PI}} \left[
            \fhat(x) \neq \fhat(x')
        \right]
        \tag{$\fhat \in \{0,1\}$} \\
    &= \underbrace{\Pr_{\substack{h_\LS \\ h_\PI}} \left[
            \fhat(x) \neq \fhat(x') ~\middle|~ h_\LS(x) = h_\LS(x')
        \right]}_{p_1}
        \cdot \Pr_{h_\LS} [h_\LS(x) = h_\LS(x')] \nonumber\\
    &\qquad
        + \underbrace{\Pr_{\substack{h_\LS \\ h_\PI}} \left[
            \fhat(x) \neq \fhat(x') ~\middle|~ h_\LS(x) \neq h_\LS(x')
        \right]}_{p_2}
        \cdot \Pr_{h_\LS} [h_\LS(x) \neq h_\LS(x')]
        \label{equation:2-level-fhat-expected-difference}
\end{align}

We evaluate $p_1$ and $p_2$ separately. First, noting that a pairwise-independent hash family is also uniform, we have
\begin{align*}
   & \quad \Pr_{h_\LS, h_\PI} \left[
            \fhat(x) = 0,
            \fhat(x') = 1
            ~\middle|~ h_\LS(x) = h_\LS(x')
        \right]\\
    &= \Pr_{h_\LS, h_\PI} \left[
            f(x) < \frac{h_\PI(h_\LS(x))}{k},
            f(x') \geq \frac{h_\PI(h_\LS(x'))}{k}
            ~\middle|~ h_\LS(x) = h_\LS(x')
        \right] \\
    &= \Pr_{h_\LS, h_\PI} \left[
            f(x) < \frac{h_\PI(h_\LS(x))}{k} \leq f(x')
            ~\middle|~ h_\LS(x) = h_\LS(x')
        \right] \\
    &= \Pr_{h_\LS, h_\PI} \left[
            f(x) < \frac{h_\PI(h_\LS(x))}{k} \leq f(x')
        \right]
        \tag{$h_\PI$ is uniform}
\end{align*}
By symmetry, $\Pr_{h_\LS, h_\PI} [\fhat(x) = 1, \fhat(x') = 0 ~|~ h_\LS(x) = h_\LS(x')] = \Pr_{h_\LS, h_\PI} [f(x) \geq \frac{h_\PI(h_\LS(x))}{k} > f(x')]$; but this equals zero, since $f(x) \leq f(x')$. Thus
\begin{align*}
    p_1
    &= \Pr_{h_\LS, h_\PI} \left[
            \fhat(x) = 1,
            \fhat(x') = 0
            ~\middle|~ h_\LS(x) = h_\LS(x')
        \right]
        + \Pr_{h_\LS, h_\PI} \left[
            \fhat(x) = 0,
            \fhat(x') = 1
            ~\middle|~ h_\LS(x) = h_\LS(x')
        \right] \\
    &= \Pr_{h_\LS, h_\PI} \left[
            f(x) < \frac{h_\PI(h_\LS(x))}{k} \leq f(x')
        \right] \\
    &= |f(x) - f(x')| \pm \frac{2}{k}
        \tag{by \cref{equation:hash-interval-bounds}}
\end{align*}

Next, to compute $p_2$, we have
\begin{align*}
    &\quad \Pr_{h_\LS, h_\PI}\left[
            \fhat(x) = 1,
            \fhat(x') = 0
            ~\middle|~ h_\LS(x) \neq h_\LS(x')
        \right] \\
    &= \Pr_{h_\LS, h_\PI}\left[
            {f}(x) \geq \frac{h_\PI(h_\LS(x))}{k},
            {f}(x')< \frac{h_{\PI}(h_{\LS}(x'))}{k}
            ~\middle|~ h_\LS(x) \neq h_\LS(x')
        \right] \\
    &= \Pr_{h_\LS, h_\PI}\left[
            {f}(x) \geq \frac{h_\PI(h_\LS(x))}{k},
            ~\middle|~ h_\LS(x) \neq h_\LS(x')
        \right]
        \cdot \Pr_{h_\LS, h_\PI}\left[
            {f}(x')< \frac{h_{\PI}(h_{\LS}(x'))}{k}
            ~\middle|~ h_\LS(x) \neq h_\LS(x')
        \right]
        \tag{$h_{\PI}$ is pairwise independent} \\
    &= f(x)(1-f(x')) \pm \frac{1}{k}
        \tag{$h_{\PI}$ is uniform}
\end{align*} 
and by symmetry, $\Pr_{h_\LS, h_\PI} [\fhat(x) = 0, \fhat(x') = 1 ~|~ h_\LS(x) \neq h_\LS(x')] = (1-f(x))f(x') \pm \frac{1}{k}$. Thus
\begin{align*}
    p_2
    &= \Pr_{h_\LS, h_\PI} \left[
            \fhat(x) = 1,
            \fhat(x') = 0
            ~\middle|~ h_\LS(x) \neq h_\LS(x')
        \right]
        + \Pr_{h_\LS, h_\PI} \left[
            \fhat(x) = 0,
            \fhat(x') = 1
            ~\middle|~ h_\LS(x) \neq h_\LS(x')
        \right] \\
    &= f(x) - 2f(x')f(x) + f(x') \pm \frac{2}{k}
\end{align*}

Substituting $p_1$ and $p_2$ back into \cref{equation:2-level-fhat-expected-difference} yields
\begin{align}
    \E_{h_\LS, h_\PI} \left[\left|
            \fhat(x) - \fhat(x')
        \right|\right]
    &= p_1 \cdot \Pr_{h_\LS} [h_\LS(x) = h_\LS(x')]
        + p_2 \cdot \Pr_{h_\LS} [h_\LS(x) \neq h_\LS(x')]
        \nonumber \\
    &= |f(x) - f(x')| \cdot (1-d(x,x'))
        + (f(x) - 2f(x')f(x) + f(x')) \cdot d(x,x')
        \pm \frac{2}{k}
        \tag{$h_\LS$ is LSH} \\
    &= |f(x) - f(x')|
        + 2f(x)(1-f(x')) \cdot d(x,x')
        \pm \frac{2}{k}
        \label{equation:exact-expected-fairness-bound} \\
    &\leq \alpha \cdot d(x,x')
        + \beta
        + 2f(x)(1-f(x')) \cdot d(x,x')
        + \frac{2}{k}
        \tag{$f$ is $(\alpha,\beta,d)$-fair} \\
    &\leq [\alpha + 2f(x)(1-f(x'))] \cdot d(x,x')
        + \beta
        + \epsilon
        \tag{$k \geq 2/\epsilon$}
\end{align}
which proves the pairwise fairness bound. The aggregate fairness bound then follows from \cref{lemma:pairwise-fairness-implies-aggregate-fairness}.

\end{proof}

\subsection{Bias-Variance Decomposition}
\label{appendix-subsection:bias-variance-decomposition}
\begin{proof}[Proof of \cref{lemma:bias-variance-decomposition}]
For any $c>0$, we have
\begin{align*}
    \left|\fhat(x) - \indicatorf(x)\right|
    &\leq \left|\E_{f,\fhat} \left[\fhat(x) - \indicatorf(x)\right]\right|
        + \left|\fhat(x) - \indicatorf(x) - \E_{f,\fhat} \left[\fhat(x) - \indicatorf(x)\right]\right| \\
    &\leq \left|\E_{f,\fhat} \left[\fhat(x) - \indicatorf(x)\right]\right|
        + c \cdot \Var_{f,\fhat} \left(\fhat(x) - \indicatorf(x) - \E_{f,\fhat} \left[\fhat(x) - \indicatorf(x)\right]\right)
        \tag{by Chebyshev's inequality, w.p. $1-1/c^2$} \\
    &\leq \left|\E_{f,\fhat} \left[
                \fhat(x) - \indicatorf(x)\right]
            \right|
        + c \cdot \Var_{\fhat} \left(
                \fhat(x) - \E_{\fhat} \left[\fhat(x)\right]
            \right)
        + c \cdot \Var_f \left(
                \indicatorf(x) - \E_f [\indicatorf(x)]
            \right)
        \tag{$\fhat(x) - \E_{\fhat}[\fhat(x)]$ and $\indicatorf(x) - \E_f[\indicatorf(x)]$ have mean zero} \\
    &\leq \left|\E_{f,\fhat} \left[
                \fhat(x) - \indicatorf(x)\right]
            \right|
        + c \cdot \Var_{\fhat} \left(\fhat(x)\right)
        + c \cdot \Var_f \left(\indicatorf(x)\right)
\end{align*}

The above calculation fails with probability at most $1/c^2$, in which case the left-hand side still obeys the simple bound $|\fhat(x) - \indicatorf(x)|\leq 1$.
Thus taking expectations of both sides, we have
\begin{align*}
    \E_{f,\fhat} \left[\left|\fhat(x) - \indicatorf(x)\right|\right]
    \leq \left|\E_{f,\fhat} \left[
                \fhat(x) - \indicatorf(x)\right]
            \right|
        + c \cdot \Var_{\fhat} \left(\fhat(x)\right)
        + c \cdot \Var_f \left(\indicatorf(x)\right)
        + \frac{1}{c^2}
\end{align*}

with probability 1 for any $c > 0$.
A choice of $c = (\Var_{\fhat \sim \calF} (\fhat(x)) + \Var_f (\indicatorf(x)))^{-1/3}$ yields the result.
\end{proof}

\subsection{Metric-Fair Derandomization Preserves Threshold Fairness}
\label{appendix-subsection:metric-fair-derandomization-preserves-threshold-fairness}
\begin{proof}[Proof of \cref{lemma:metric-fair-derandomization-preserves-threshold-fairness}]
First, fix some $\sigma \in (0,1)$ and let $X^2_\leqsigma := \left\{(x,x') \in X^2 ~\middle|~ d(x,x') \leq \sigma\right\}$. Observe the following translations between metric and threshold fairness on this set:
\begin{enumerate}
    \item If $f$ is $(\sigma,\tau,d)$-threshold fair, then for any $(x,x') \in X^2_\leqsigma$,
        \begin{align*}
            |f(x) - f(x')|
            \leq \tau
            = 0 \cdot d(x,x') + \tau
        \end{align*}
        So, $f$ is also $(0,\tau,d)$-metric fair on such pairs $(x,x')$.
    
    \item If $f$ is $(\alpha,\beta,d)$-metric fair on all $(x,x') \in X^2_\leqsigma$, then for such pairs,
        \begin{align*}
            |f(x) - f(x')|
            \leq \alpha \cdot d(x,x') + \beta
            \leq \alpha \sigma + \beta
        \end{align*}
        So, $f$ is also $(\sigma, \alpha \sigma + \beta, d)$-threshold fair.
\end{enumerate}
Now suppose we run our derandomization procedure on a $(\sigma,\tau,d)$-threshold fair stochastic classifier $f$. Let $\calF$ be the deterministic classifier family from which we sample our output. Then $f$ is $(0,\tau,d)$-metric fair over $X^2_\leqsigma$ (by observation 1 above), $\calF$ is then $(A(0),B(\tau),d)$-metric fair over $X^2_\leqsigma$ (by the fairness preservation guarantee), and $\calF$ is also $(\sigma, A(0) \cdot \sigma + B(\tau), d)$-threshold fair (by observation 2).
\end{proof}

\begin{proof}[Proof of \cref{corollary:threshold-fairness-preserving-derandomizations}]
If $f$ is $(\sigma,\tau,d)$-threshold fair, then $\calF_\LS$ is $(\sigma,\tau',d)$-threshold fair, where
\begin{align*}
    \tau'
    &= A(0) \cdot \sigma
        + B(\tau)
        \tag{\cref{lemma:metric-fair-derandomization-preserves-threshold-fairness}} \\
    &= \frac{1}{2} \cdot \sigma
        + \tau
        + \frac{2}{k}
        \tag{\cref{corollary:LSH-derandomization-worst-case-fairness}} \\
    &= \sigma + \tau
        \tag{choice of $k \geq 4/\sigma$}
\end{align*}
\end{proof}

\subsection{Pairwise Fairness Implies Aggregate Fairness}
\label{appendix-subsection:pairwise-fairness-implies-aggregate-fairness}
\begin{proof}[Proof of \cref{lemma:pairwise-fairness-implies-aggregate-fairness}]
For all distances $\xi \in [0,1]$, let $X^2_\xi := \left\{(x,x') \in X^2 ~\middle|~ d(x,x') = \xi\right\}$ denote the set of point pairs at distance exactly $\xi$. Then, for any given $\fhat \in \calF$, let
\begin{align*}
    \rho_\xi(\fhat)
        := \Pr_{(x,x') \sim X^2_\xi} \left[
                \fhat(x) \neq \fhat(x')
            \right]
    \qquad \text{and} \qquad
    \rho_{\leqtau}(\fhat)
        := \Pr_{(x,x') \sim X^2_\leqtau} \left[
                \fhat(x) \neq \fhat(x')
            \right]
\end{align*}
denote the fraction of pairs at distance $\xi$ and within $\tau$, respectively, to which $\fhat$ assigns different outputs. Treating $\rho_\xi(\fhat)$ as a random variable of $\fhat$, we have
\begin{align}
    \label{equation:expected-fraction-of-separated-pairs-at-distance-xi}
    \E_{\fhat \sim \calF} \left[\rho_\xi(\fhat)\right]
    = \E_{\fhat \sim \calF} \left[
        \Pr_{\substack{(x,x') \\ \sim X^2_\xi}} \left[
            \fhat(x) \neq \fhat(x')
        \right]\right]
    = \E_{\fhat \sim \calF} \left[
        \E_{\substack{(x,x') \\ \sim X^2_\xi}} \left[
            \left|\fhat(x) - \fhat(x')\right|
        \right]\right]
    = \E_{\substack{(x,x') \\ \sim X^2_\xi}} \left[
        \E_{\fhat \sim \calF} \left[
            \left|\fhat(x) - \fhat(x')\right|
        \right]\right]
\end{align}

Thus the fraction of separated pairs within distance $\tau$ is
\begin{align}
    \E_{\fhat \sim \calF} \left[\rho_{\leqtau}(\fhat)\right]
    &:= \E_{\fhat \sim \calF} \left[
            \Pr_{(x,x') \sim X^2_\leqtau} \left[
                \fhat(x) \neq \fhat(x')
            \right]
        \right]
        \nonumber \\
    &= \int_0^\tau \E_{\fhat \sim \calF} \left[
            \Pr_{(x,x') \sim X^2_\leqtau} \left[\fhat(x) \neq \fhat(x') ~\middle|~ d(x,x') = \xi\right]
            \cdot \Pr_{(x,x') \sim X^2_\leqtau} [d(x,x') = \xi] \ d\xi
        \right]
        \nonumber \\
    &= \int_0^\tau \E_{\fhat \sim \calF} \left[
            \Pr_{(x,x') \sim X^2_\xi} \left[\fhat(x) \neq \fhat(x')\right]
        \right]
        \cdot \Pr_{(x,x') \sim X^2_\leqtau} [d(x,x') = \xi] \ d\xi
        \nonumber \\
    &= \int_0^\tau \E_{(x,x') \sim X^2_\xi} \left[
            \E_{\fhat \sim \calF} \left[
                \left|\fhat(x) - \fhat(x')\right|
            \right]
        \right]
        \cdot \Pr_{(x,x') \sim X^2_\leqtau} [d(x,x') = \xi] \ d\xi
        \qquad \qquad \qquad \text{(by \cref{equation:expected-fraction-of-separated-pairs-at-distance-xi})}
        \label{equation:expected-fraction-of-separated-close-pairs} \\
    &\leq \int_0^\tau (\alpha \xi + \beta)
        \Pr_{(x,x') \sim X^2_\leqtau} [d(x,x') = \xi] \ d\xi
        \qquad \qquad \qquad \qquad \qquad \qquad \text{(by $(\alpha,\beta,d)$-fairness)}
        \label{equation:expected-fraction-of-separated-close-pairs-for-fair-classifier} \\
    &\leq (\alpha \tau + \beta) \int_0^\tau
        \Pr_{(x,x') \sim X^2_\leqtau} [d(x,x') = \xi] \ d\xi
        \nonumber \\
    &= \alpha \tau + \beta
        \label{equation:expected-fraction-of-separated-tau-close-pairs}
\end{align}

Since $\rho_\leqtau \in [0,1]$, $\Var(\rho_\leqtau) = \E[\rho_\leqtau^2] - \E[\rho_\leqtau]^2 \leq \E[\rho_\leqtau]$. Thus applying Chebyshev's inequality to \cref{equation:expected-fraction-of-separated-tau-close-pairs} yields
\begin{align*}
    \Pr_{\fhat \sim \calF} \left[
            \rho > \left(1 + \frac{1}{\sqrt{\delta}}\right) (\alpha\tau + \beta)
        \right]
    \leq \Pr_{\fhat \sim \calF} \left[
            \rho > \left(1 + \frac{1}{\sqrt{\delta}}\right) \E_{\fhat \sim \calF} [\rho]
        \right]
    \leq \delta
\end{align*}
which proves the claim.
\end{proof}

\subsection{Output Approximation and Loss Approximation}
\label{appendix-subsection:output-approximation-and-loss-approximation}
\begin{proof}[Proof of \cref{lemma:output-approximation-implies-loss-approximation}]
For any $x \in X$ and $y \in \{0,1\}$,
\begin{align*}
    \E_{\fhat \sim \calF} \left[
            L(\fhat,x,y)
        \right]
    &= \E_{\fhat \sim \calF} \left[
            \ell(\fhat(x),y)
        \right]
        \tag{$\fhat(x) \in \{0,1\}$} \\
    &= \E_{\fhat} \left[
            \ell(\fhat(x),y)
            ~\middle|~
            \fhat(x) = 1
        \right]
        \cdot \Pr_{\fhat} \left[\fhat(x) = 1\right]
        + \E_{\fhat} \left[
            \ell(\fhat(x),y)
            ~\middle|~
            \fhat(x) = 0
        \right]
        \cdot \Pr_{\fhat} \left[\fhat(x) = 0\right] \\
    &= \ell(1,y) \cdot \E_{\fhat} \left[\fhat(x)\right]
        + \ell(0,y) \cdot \left(1 - \E_{\fhat} \left[\fhat(x)\right]\right) \\
    &= \ell(1,y) f(x)
        + \ell(0,y) \left(1 - f(x)\right)
        \pm \bias(\fhat,f,x) \\
    &= f(x)\ell(1,y) + (1-f(x))\ell(0,y)
        \pm \bias(\fhat,f,x)
\end{align*}
which proves the first inequality concerning the bias. For the variance, notice that since $\ell$ is binary, either $\Var_{\fhat} \left(\ell(\fhat(x),y)\right) = \Var_{\fhat} \left(\fhat(x)\right)$ or $\Var_{\fhat} \left(\ell(\fhat(x),y)\right) = 0$.
\end{proof}
\section{Manipulation Deterrence in Strategic Classification}
\label{appendix-section:manipulation-deterrence-in-strategic-classification}
Fair derandomization procedures carry implications for the \emph{strategic classification} problem, a popular framework for modeling the behavior of self-interested agents subject to classification decisions \cite{hardt2016strategic,cai2015optimum,chen2018strategyproof,dong2018strategic,chen2020learning}. Formally, strategic classification is a Stackelberg game, or a sequential game between two players:
\begin{enumerate}
    \item First, a \emph{decision maker} or \emph{model designer} publishes a classifier. Traditionally, this means a stochastic classifier $f: X \rightarrow [0,1]$, but in our setting, the model designer may publish a family of deterministic classifiers $\calF$, and promises to select a single classifier from $\calF$ uniformly at random.
    
    \item Next, a \emph{strategic agent} or \emph{decision subject}, who is associated with some feature vector $x \in X$, decides either to present their true features $x$, or to change or \emph{manipulate} their features to some $x' \in X$ to obtain the favorable outcome $\hat{f}(x') = 1$ with higher probability. However, the agent incurs a cost $c(x,x') \geq 0$ for altering their features.
\end{enumerate}
Given a (stochastic or deterministic) classifier $f: X \to [0,1]$ and cost function $c: X^2 \to [0,1]$, the \emph{utility} of an agent with original features $x$ who changes to $x'$ is defined as
\begin{align*}
    U_f(x,x') := f(x') - c(x, x')
\end{align*}
and the utility-maximizing move $\Delta_f(x) := \argmax_{x' \in X} U_f(x,x')$ is called the \emph{best response} of $x$ under $f$ and $c$.

In the following proposition, we observe a general connection between metric fairness and strategic manipulation; namely that the more fair a classifier is with respect to a metric cost function, the less incentive agents have to manipulate their features. The reason is intuitive: if a classifier is a smooth function, then an agent $x$ cannot expect their outcome to change much by moving to some nearby point $x'$.

\begin{proposition}[Metric fairness implies reduced manipulation incentive]
\label{proposition:metric-fairness-implies-reduced-manipulation-incentive}
Let $c$ be a metric cost function and let $f$ be a $(\alpha,\beta,c)$-metric fair classifier. Then the maximum utility gained by manipulating $x$ to $x'$ is
\begin{align*}
    U_f(x,x') - U_f(x,x)
    \leq (\alpha-1) \cdot c(x,x') + \beta .
\end{align*}
If $f$ is a deterministic classifier drawn from a family $\calF$, then this holds in expectation over the sampling of $f$.
\end{proposition}

\begin{proof}[Proof of \cref{proposition:metric-fairness-implies-reduced-manipulation-incentive}]
Under a classifier $f$, an individual with original features $x \in X$ who changes to $x' \in X$ derives utility
\begin{align*}
    U_f(x,x')
    &= f(x') - c(x,x') \\
    &\leq f(x) + |f(x') - f(x)| - c(x,x') \\
    &\leq f(x) + \alpha \cdot c(x,x') + \beta - c(x,x')
        \tag{$f$ is $(\alpha,\beta,c)$-fair} \\
    &= f(x) + (\alpha-1) \cdot c(x,x') + \beta \\
    &= U_f(x,x) + (\alpha-1) \cdot c(x,x') + \beta
\end{align*}
which proves the claim for stochastic classifiers. The proof for a deterministic family $\calF$ results from taking an expectation $\E_{f \sim \calF} [\cdot]$ on both sides.
\end{proof}

Braverman and Garg \cite{braverman2020role} already observed this fact for a stochastic classifier with $\alpha=1$ and $\beta=0$, in which case there is no incentive to manipulate. Note that by \cref{proposition:fairness-is-impossible-for-finite-deterministic-families}, deterministic families cannot achieve such small fairness parameters; hence the upper bound of \cref{proposition:metric-fairness-implies-reduced-manipulation-incentive} cannot rule out \emph{some} incentive to manipulate. Nevertheless, it presents a nontrivial worst-case guarantee since, for a classifier without any fairness constraints, there may be individuals near the decision boundary who can flip their decision from, for example, $f(x)=0$ to $f(x')=1$ at near-zero cost, thereby gaining utility $U(x,x') - U(x,x) \approx 1$ through manipulation.

Cost functions studied in the strategic classification literature include the $L_2$ \cite{hardt2016strategic,bruckner2011stackelberg} and Mahalanobis \cite{chen2021linear} distances, both of which are metrics with known LSH families \cite{andoni2006near,jain2008fast}. Therefore, stochastic classifiers trained to be fair with respect to these costs automatically reduce incentives to manipulate features, and if such classifiers are derandomized using fairness-preserving methods, this quality is probably approximately preserved.

\newpage

\end{document}